%% file: main.tex
\documentclass[conference]{IEEEtran}
\IEEEoverridecommandlockouts

\usepackage{multirow, bigstrut, cellspace}

\usepackage{makecell}
\usepackage{subcaption}
\usepackage{mathrsfs}  
\usepackage{siunitx}
\usepackage{amsthm}
\usepackage{dsfont}

\PassOptionsToPackage{noend}{algpseudocode}
\PassOptionsToPackage{hyphens}{url}
\usepackage{url}

\usepackage{mystyle}
\usepackage{list_macro}

\newcolumntype{C}[1]{>{\centering\arraybackslash}m{#1}}

\title{
    Imitation Learning from Observations: An Autoregressive Mixture of Experts Approach
    \thanks{Work supported by
    the Research Foundation Flanders (FWO) postdoctoral grant 12Y7622N and research projects
    G081222N, G033822N, and G0A0920N; 
    Research Council KU Leuven C1 project No. C14/24/103; 
    European Union’s Horizon 2020 research and innovation programme under the Marie Skłodowska-Curie grant agreement No. 953348;
    Flemish Agency for Innovation and Entrepreneurship (VLAIO) under research project No. HBC.2021.0939 (BECAREFUL).
    }
    \thanks{
        $^1$ KU Leuven, Department of Electrical Engineering \textsc{esat-stadius} -- %
        Kasteelpark Arenberg 10, bus 2446, B-3001 Leuven, Belgium
        \newline
        {\sf
            \{%
                \href{mailto:renzi.wang@kuleuven.be}{renzi.wang},
                \href{mailto:panos.patrinos@kuleuven.be}{panos.patrinos}%
            \}%
            \href{mailto:renzi.wang@kuleuven.be,panos.patrinos@kuleuven.be}{@kuleuven.be}%
        }
    }
    \thanks{
        $^2$ Siemens Digital Industries Software, 3001 Leuven, Belgium
        \newline
        {\sf
            \{%
                \href{mailto:flavia.acerbo@siemens.com}{flavia.acerbo},
                \href{mailto:son.tong@siemens.com}{son.tong}%
            \}%
            \href{mailto:flavia.acerbo@siemens.com,son.tong@siemens.com}{@siemens.com}%
        }
    }
}
\author{
    \IEEEauthorblockN{Renzi Wang$^1$}
\and
    \IEEEauthorblockN{Flavia Sofia Acerbo$^2$}
\and 
    \IEEEauthorblockN{Tong Duy Son$^2$}
\and 
    \IEEEauthorblockN{Panagiotis Patrinos$^1$}
}

\begin{document}

\maketitle

\begin{abstract}
\input{contents/abstract.tex}
\end{abstract}
\input{contents/introduction.tex}

\input{contents/notation.tex}

\section{Problem Formulation}\label{sec: problem_formulation}
\input{contents/problem_formulation/problem_formulation.tex}
\input{contents/problem_formulation/difficulty_with_a_general_mle_formulation.tex}

\section{Two-stage approach}\label{sec: two-stage-approach}
\input{contents/two_stage_approach/two_stage_approach.tex}
\input{contents/two_stage_approach/policy_learning.tex}

\section{Stable system identification}\label{sec: stability}
\input{contents/extensions/stable_system_identification.tex}

\section{Application: policy learning using driving data from human demonstration}\label{sec: application}
\input{contents/autonomous_driving_application/introduction.tex}

\subsection{Control input estimation}\label{sec: input_estimation_dynamic_bicycle}

\input{contents/autonomous_driving_application/control_input_estimation.tex}

\subsection{Numerical experiment result}
\input{contents/autonomous_driving_application/numerical_experiments/introduction_and_metric.tex}

\input{contents/autonomous_driving_application/numerical_experiments/lane_keeping.tex}
\input{contents/autonomous_driving_application/numerical_experiments/double_lane_change.tex}

\section{Conclusion}
\input{contents/conclusion.tex}

\addcontentsline{toc}{section}{References}
\bibliographystyle{IEEEtran}
\bibliography{IEEEabrv,reference}

\end{document}

%% file: contents/abstract.tex
This paper presents a novel approach to imitation learning from observations, 
where an autoregressive mixture of experts model is deployed to fit the underlying policy. 
The parameters of the model are learned via a two-stage framework. 
By leveraging the existing dynamics knowledge, 
the first stage of the framework estimates the control input sequences and hence reduces the problem complexity. 
At the second stage, 
the policy is learned by solving a regularized maximum-likelihood estimation problem using the estimated control input sequences.
We further extend the learning procedure by incorporating a Lyapunov stability constraint to ensure asymptotic stability of the identified model, 
for accurate multi-step predictions.
The effectiveness of the proposed framework is validated using two autonomous driving datasets collected from human demonstrations, 
demonstrating its practical applicability in modelling complex nonlinear dynamics.

%% file: contents/introduction.tex
\section{Introduction}
In modern control frameworks, such as model predictive control (MPC), 
accurate prediction of future system states is essential for effective control. 
This prediction task becomes particularly challenging in multi-agent environments with limited or no communication between agents, 
e.g. mixed human and autonomous driving scenarios,
where each controller has restricted access to other agents' control policies, 
and yet inaccurate predictions of their future states could cause severe consequences in autonomous driving.

The majority of motion prediction methods either use simple models, 
such as constant velocity or acceleration, or rely on learning-based approaches to predict future state trajectories \cite{gulzar2021survey}. 
However, in certain safety-critical interactive scenarios, it can also be useful to predict the control actions of the other vehicles \cite{wang2023interaction}. 
This enables anticipation of their intentions and allows the integration of their dynamics into the optimization process, enabling the ego vehicle to respond more quickly.
To achieve accurate predictions in such scenarios, 
one must construct models of control policies from available observational data. 
This approach of learning control policies solely from trajectory observations falls under the domain of Imitation Learning from Observations (IfO) \cite{ijcai2019p0882}.
The absence of control action information in the demonstrations prevents the use of typical imitation learning approaches like behavioral cloning (BC) \cite {bain1995framework} or generative adversarial imitation learning (GAIL) \cite{ho2016generative}. Research on adapting these algorithms to the IfO setting can be categorized into model-free and model-based approaches. 
Model-free approaches typically work in an adversarial setting, 
where a discriminator is trained to distinguish between state transitions generated by the policy or from the demonstrations \cite{torabi2019generativeadversarialimitationobservation}. 
On the other hand, 
model-based approaches exploit the hierarchical structure of the closed-loop trajectories by separating the problem into model learning and policy learning \cite{torabi2018behavioral, zhang2024action, nair2017combining, pmlr-v97-edwards19a}.
In these approaches, a forward or an inverse dynamical model is first learned from system data. 
This model is used to infer the control actions from the demonstrated trajectories.
Policy learning then proceeds using these inferred actions.
However,
all these approaches are purely data-driven, 
and hence overlook at integrating possible prior knowledge of the system dynamics and defining a control policy structure which can be used as part of MPC frameworks, 
such that it can provide fast and stable control predictions over a time horizon.
In contrast to those approaches, 
we utilize the prior knowledge of the system model instead of learning the model from scratch using data.

Previous works considered estimating optimal control policies for the autonomous driving agents through inverse optimal control \cite{kuderer2015learning, even2022learning}. 
However, this approach has two main limitations: 
it restricts the possible control policies to those derived from a pre-defined optimal control problem,
and the resulting formulations are computationally expensive for real-time implementation within an MPC framework. 
Switching systems offer an attractive alternative by providing a balanced model that can approximate nonlinear systems with simpler computational requirements compared to online optimization approaches.
Notably, \cite{bemporad2002explicit} has shown that piecewise affine systems---a special class of switching systems---represent the explicit solution of MPC for discrete-time linear time-invariant systems.
This theoretical connection further motivates the use of switching systems for approximating general control laws.
In this work, we focus on \emph{stochastic switching systems},
where the stochastic model explicitly accounts for the inherent uncertainty in agent behavior.

Like any model employed in control applications,
a critical challenge lies in ensuring reliable multi-step predictions.
Prediction errors tend to accumulate over time due to model mismatch, 
potentially leading to significant deviations in longer horizons. 
Two primary approaches have emerged to address this issue:
\begin{inlinelist}
    \item Directly train the prediction model with loss measuring the multi-step-ahead prediction error, such as in \cite{beintema2023continuoustime, forgione2021continuous};
    \item introduce a stability constraint in the model.
\end{inlinelist}
While the first approach directly addresses multi-step prediction accuracy, 
it remains inherently limited by the specific prediction horizon used during training. 
When deployments require predictions beyond this training horizon, 
error accumulation issues persist. 
The second approach, focusing on stability guarantees, 
has been explored across various data-driven modeling frameworks:
\cite{bonassi2021stability} derived sufficient conditions for Neural Nonlinear Auto Regressive eXogenous (NNARX).
\cite{miller2018stable} studied the stability property for recurrent neural network,
and \cite{bonassi2020lstm} established a sufficient condition for long short-term memory (LSTM) networks.

Our work makes the following key contributions:
\begin{itemize}
    \item We introduce a stochastic switching system to approximate the control policy from observations. The parameters in the policy can be learned via a two-stage approach;
    \item We derive a sufficient Lyapunov stability condition for the proposed stochastic switching model. Such condition can be directly enforced during the training process;
    \item We validate the effectiveness of the proposed approach on two autonomous driving datasets containing human demonstrations.
\end{itemize}

The paper is organized as follows:
\cref{sec: problem_formulation} formulates the problem of policy learning from state-only data.
\cref{sec: two-stage-approach} introduces the proposed model and a two-stage learning approach.
After that, \cref{sec: stability} establishes a sufficient stability condition for the identified model.
Finally, \cref{sec: application} evaluates the proposed method on two autonomous driving datasets.

%% file: contents/notation.tex
\section*{Notation}
Denote the cone of $m \times m$ positive definite matrices by $\mathbb{S}_{++}^{m}$.
We define $\lse: \R^{d} \to \R$ as
\(
    \lse(x) \! =\! \ln \big(
        \textstyle{\sum_{j=1}^{d}} \exp(x_j) 
    \big),
\)
and the softmax function $x \mapsto \sigma(x)$ 
with  
\(
    \sigma_i(x) \! = \!\exp\big(x_i\big) / \sum_{j=1}^{d}\exp\big(x_j\big)
\)
the $i$th entry of the output $\sigma(x)$.

%% file: contents/problem_formulation/problem_formulation.tex
Let $\stateSeq \dfn \{x_t\}_{t = 0}^{T-1}$ be the state trajectory generated from a known discrete-time dynamical system
\begin{equation}\label{eq: sys_dyn}
    x_{t+1} = f(x_t, u_t),
\end{equation}
where $x_t \in \re^{n_x}$ is the system state and $u_t \in \re^{n_u}$ is the control input,
which is sampled from a stochastic feedback control policy 
\begin{equation}\label{eq: control_policy}
    u_t \sim \pi(x_t, \dots, x_{t-t_x}, u_{t-1}, \dots, u_{t-t_u}, \xi_t \midsc \theta)
\end{equation}
where $t_x \geq 0$, $t_u > 0$ denotes the window length with respect to the system states and the previous control inputs, respectively.
Variable $\xi_t \in \Xi$ represents the latent random variable,
and $\theta \in \mathcal{X}_{\theta}$ is the policy parameter, where $\mathcal{X}_{\theta}$ is the parameter space.

In this work, 
we are interested in approximating the control policy $\pi$
using the state trajectory $\stateSeq$.

%% file: contents/problem_formulation/difficulty_with_a_general_mle_formulation.tex
\subsection{Challenges in direct maximum likelihood estimation}\label{sec: challenge}
The policy approximation problem can be formulated as a regularized maximum likelihood estimation (MLE) problem (where we minimize the negative log-likelihood):
\begin{equation}\label{eq: mle_general}
    \minimize_{\theta} \ell(\theta) + \reg(\theta)
\end{equation}
Here, the negative log-likelihood is defined as
\begin{equation*}
    \ell(\theta) = -\ln p(\stateSeq \midsc \theta)
\end{equation*}
and $\reg: \mathcal{X}_{\theta} \to \re_+$ is a regularization term. 
Let $\inputSeq \dfn \{u_t\}_{t=0}^{T-1}$ be the unobserved sequence of control inputs, 
and $\modeSeq \dfn \{\xi_t\}_{t=0}^{T-1}$ be the unobserved sequence of random variables, both corresponding to the state trajectory $\stateSeq$.
The likelihood can be expressed as
\begin{equation}\label{eq: likelihood_general}
    p(\stateSeq \midsc \theta) = \int_{\inputSeq} \int_{\modeSeq \in \Xi^{T}} p(\stateSeq, \inputSeq, \modeSeq \midsc \theta) \,d \inputSeq\, d\modeSeq.
\end{equation}

A common approach to solve an MLE problem with latent variables is to apply the expectation-maximization (EM) method \cite{dempster1977maximum}.
Given the posterior distribution 
\begin{equation}\label{eq: posterior_general}
    p(\inputSeq, \modeSeq \mid \stateSeq \midsc \tilde{\theta}) 
    = \frac{
        p(\inputSeq, \modeSeq, \stateSeq \midsc \tilde{\theta})
    }{
        \int_{\inputSeq}\int_{\modeSeq \in \Xi^T}p(\inputSeq, \modeSeq, \stateSeq \midsc \tilde{\theta}) \,d \inputSeq \,d \modeSeq
    }
\end{equation}
with parameter $\tilde{\theta}$, 
one can apply Jensen's inequality to the negative log-likelihood:
\begin{align} 
    \ell(\theta) 
    = &\; \nonumber {-}\ln \Big[\int_{\inputSeq} \int_{\modeSeq \in \Xi^{T}} p(\stateSeq, \inputSeq, \modeSeq \midsc \theta) \,d \inputSeq\, d\modeSeq\Big] \\
    = &\; \nonumber {-}\ln \left[\int_{\inputSeq} \int_{\modeSeq \in \Xi^{T}} p(\stateSeq, \inputSeq, \modeSeq \midsc \theta) \frac{p(\inputSeq, \modeSeq \mid \stateSeq \midsc \tilde{\theta})}{p(\inputSeq, \modeSeq \mid \stateSeq \midsc \tilde{\theta})} \,d \inputSeq\, d\modeSeq\right] \\
    \leq &\; {-} \int_{\inputSeq} \!\int_{\modeSeq \in \Xi^{T}} \!p(\inputSeq, \modeSeq \mid \stateSeq \midsc \tilde{\theta}) \!\ln \left[\frac{p(\stateSeq, \inputSeq, \modeSeq \midsc \theta)}{p(\inputSeq, \modeSeq \mid \stateSeq \midsc \tilde{\theta})}\right] \,d \inputSeq\, d\modeSeq.
    \label{eq: upper_bound_general} 
\end{align}
The right-hand side of the inequality \eqref{eq: upper_bound_general} computes an expectation.
In the maximization step, 
owing to the negative sign, 
one minimizes the regularized negative expectation with parameter $\tilde{\theta}$ equal to the solution from previous iteration. 
However, for general random variables $\inputSeq$, $\modeSeq$,
computing the posterior distribution \eqref{eq: posterior_general}
and the expectation in \eqref{eq: upper_bound_general} can be intractable due to the integration.

%% file: contents/two_stage_approach/two_stage_approach.tex
To address the aforementioned difficulty,
we propose a two-stage approach. 
First, instead of integrating the latent control input $\inputSeq$,
we find a point estimation 
$\bar{u}_t$, for $t = 0, \dots, T-1$, 
by utilizing the state sequence and the prior knowledge on system dynamics \eqref{eq: sys_dyn}.
Second, we approximate the control policy \eqref{eq: control_policy} using the following stochastic switching system \cite{wang2024em++}:
\begin{subequations}\label{eq: policy_estimation}
    \begin{align}
        \xi_t \sim &\; p(\xi_{t} \mid z_t, \xi_{t-1} = i \midsc \boldsymbol{\Theta}) = \sigma_{\xi_t}(\Theta_i^\top z_t)\label{eq: switching_mechanism}\\
          u_t =    &\; K_{\xi_t} z_t + b_{\xi_t} + w_t.\label{eq: subsystem_dyn}
    \end{align}
\end{subequations}
This system switches among $d$ subsystems.
The discrete random variable $\xi_t \in \Xi = \{1, \dots, d\}$ denotes the active subsystem index and is sampled from a categorical distribution modeled by a softmax function.
The variable $z_t$ denotes the input-state history
\begin{equation}\label{eq: input_state_history}
    z_t \dfn (u_{t-1}, \dots, u_{t-t_u}, x_t, \dots, x_{t-t_x}).
\end{equation}
To account for the underlying nonlinear dynamics that are not captured by a linear model,
the feedback control input in \eqref{eq: subsystem_dyn} is subjected to an additive Gaussian noise $w_t \sim \gauss(0, \Sigma_{\xi_t})$.
The system \eqref{eq: policy_estimation} is closely related to the mixture of experts (MoE) architecture \cite{jordan1994hierarchical},
where the softmax function \eqref{eq: switching_mechanism} acts as a gating function determining the activation probability of each subsystem (expert).
Each subsystem \eqref{eq: subsystem_dyn} acts as an expert providing a linear approximation of the control policy.
The system \eqref{eq: policy_estimation} adaptively selects experts based on the current input-state history $z_t$, 
and previous selection $\xi_{t-1}$, enabling specialized handling for different conditions.

The general switching mechanism \eqref{eq: switching_mechanism} encompasses three special cases in the literature,
where each mechanism depends on a distinct subset of these variables:
\begin{enumerate}
    \item Static switching:
    \begin{equation}\label{eq: static_switching}
        p(\xi_{t+1} \!\mid\! z_t, \xi_t\!\midsc \Theta) = p(\xi_{t+1} \!\midsc \Theta) = \sigma_{\xi_{t+1}}(\Theta)  
    \end{equation}
    with $\Theta_i = \Theta \in \re^{1 \times d}$. 
    This formulation is commonly used in mixture models such as Gaussian mixture models \cite[\S 9.2]{bishop2006pattern}.
    \item Mode-dependent switching:
    \begin{equation}\label{eq: only_mode}
        p(\xi_{t+1} \!\mid\! z_t, \xi_t\!\midsc \Theta) = p(\xi_{t+1} \!\mid \!\xi_t \!\midsc \Theta) = \sigma_{\xi_{t+1}}(\Theta_{\xi_t})  
    \end{equation}
    with $\Theta_i \in \re^{1 \times d}$ for all $i \in \Xi$.
    This formulation is commonly used in Markov jump systems \cite{costa2005discrete}.
    \item State-dependent switching:
    \begin{equation}\label{eq: only_state}
        p(\xi_{t+1} \!\mid\! z_t, \xi_t\!\midsc \Theta) = p(\xi_{t+1} \!\mid\! z_t\! \midsc \Theta) = \sigma_{\xi_{t+1}}(\Theta^\top\! z_t)
    \end{equation}
    with $\Theta_i = \Theta \in \re^{ n_z \times d}$.
    This formulation is commonly used in the classic mixture of experts models \cite{jordan1994hierarchical}.
\end{enumerate}

For the point estimation step to be well-defined, 
we consider a family of systems $f$ that satisfies the following assumption:
\begin{assumption}\label{assum: invertibility}
    For a given control input set $U \subseteq \re^{n_u}$, 
    consider the system dynamics model \eqref{eq: sys_dyn} with $u_t \in U$ for all $t = 0, \dots, T-1$.
    The mapping $f(x_t, \cdot): U \to \re^{n_x}$ is bijective for all $x_t \in \re^{n_x}$.
    Therefore, there exists a unique inverse mapping $f^{-1}(x_t, \cdot): \re^{n_x} \to U$ such that
    \[
        u_t = f^{-1}(x_t, x_{t+1})
    \]
    for any consecutive state pair $(x_t, x_{t+1})$ observed in the state trajectory.
\end{assumption}
Under this assumption,
we obtain point estimations $\bar{u}_t$, for $t = 0, \dots, T-1$.
These estimations enable us to reformulation the problem:
instead of solving the MLE problem \eqref{eq: mle_general} with likelihood \eqref{eq: likelihood_general}, 
we solve 
\begin{equation}\label{eq: mle_problem}
    \minimize_{\theta} \likelihood(\theta) + \reg(\theta)
\end{equation}
where the parameter $\theta \dfn \{K_1, b_1, \Sigma_1, \dots, K_d, b_d, \Sigma_d, \boldsymbol{\Theta}\}$,
and the likelihood is defined with the point estimations
\begin{equation}\label{eq: likelihood_problem}
    \begin{aligned}
    \likelihood(\theta) 
    = &\; -\ln p(\bar{\stateSeq}, \bar{\inputSeq}, z_{\tau} \midsc \theta) \\
    = &\; -\ln \Big[
        \tlsum_{\bar{\modeSeq} \in \Xi^{T - \tau}} 
        p(\bar{\stateSeq}, \bar{\inputSeq}, \bar{\modeSeq}, z_{\tau} \midsc \theta)
    \Big],
    \end{aligned}
\end{equation}
where the first $\tau$ data points are used for constructing the initialization
\[
    z_{\tau} = (\bar{u}_{\tau-1}, \dots, \bar{u}_{\tau-t_u}, x_{\tau}, \dots, x_{\tau-t_x}),
\]
and 
$\bar{\stateSeq} \dfn \{x_t\}_{t=\tau+1}^{T-1}$,
$\bar{\inputSeq} \dfn \{\bar{u}_t\}_{t=\tau}^{T-1}$,
$\bar{\modeSeq} \dfn \{\xi_t\}_{t=\tau}^{T-1}$.

The complete procedure is summarized in \cref{alg: two_stage}.
\begin{algorithm}[!htpb]
    \caption{Two-stage approach} \label{alg: two_stage}
    \begin{algorithmic}[1]
        \Require{State sequence $\stateSeq$, dynamical model $f$, initial guess $\theta^0$}
        \State Estimate the control input sequence by inverting the dynamical model: 
        \begin{equation}\label{eq: nonlinear_ls}
            \bar{u}_t = f^{-1}(x_t, x_{t+1}), \quad \forall t = 0, \dots, T-1.
        \end{equation} 
        \State{Fitting the model \eqref{eq: policy_estimation} by solving \eqref{eq: mle_problem} 
            using $\stateSeq$, $\bar{u}_0, \dots, \bar{u}_{T-1}$, and $\theta^0$} 
    \end{algorithmic}
\end{algorithm}

%% file: contents/two_stage_approach/policy_learning.tex
\subsection{Policy learning using \EMpp}\label{sec: policy_learning}
To solve the regularized MLE problem \eqref{eq: mle_problem},
we employ \EMpp \cite{wang2024em++}.
\EMpp applies the majorization-minimization (MM) principle \cite{lange2016mm}
to identify stochastic switching systems through solving a regularized MLE problem \eqref{eq: mle_problem}
with the negative log-likelihood $\loss: \mathcal{X}_{\theta} \to \R$
\begin{equation*}
	\begin{aligned}
		\loss(\theta)
		= {-} \ln
		p(\ve{y}, z_0  \midsc \theta)
		= {-} \ln 
		\left[ \textstyle{\sum_{\modeSeq\in \Xi^{T}}} 
		p(\ve{y}, \ve{\xi}, z_0  \midsc \theta)
		 \right],
	\end{aligned}
\end{equation*}
where $\ve{y} = \{y_t\}_{t = 0}^{T-1}$ is a sequence of observations generated with initialization $z_0$,
and $\modeSeq$ is the sequence of discrete latent variable indicating the active subsystem.

Following the MM principle, at each iteration $k$,
\EMpp constructs a convex function $\surrogate^k$ satisfying
\begin{subequations}\label{eq: property_Q}
	\begin{align}
		\surrogate^k(\theta) \geq & \; \loss(\theta), \quad \forall \theta \in \mathcal{X}_\theta,  \label{eq: Q_pos} \\
		\surrogate^k(\theta^k) =      & \; \loss(\theta^k) \label{eq: Q_equal_func}
	\end{align}
\end{subequations}
in the majorization step,
then minimizes the regularized function $\surrogate(\theta) + \reg(\theta)$ as a surrogate problem in the minimization step. 
This general framework includes the classical EM algorithm as a special case 
when the surrogate function is obtained by applying Jensen's inequality with the posterior distribution,
as presented in \cref{sec: challenge}.

Having formulated the policy learning problem \eqref{eq: mle_problem} with likelihood \eqref{eq: likelihood_problem},
we now show how to solve it using \EMpp.
The key observation is that the problem \eqref{eq: mle_problem} naturally maps to \EMpp's framework by identifying
that the state and point estimation pair $(\bar{\stateSeq}, \bar{\inputSeq})$ represents the observation sequence $\ve{y}$.
To construct the likelihood \eqref{eq: likelihood_problem}, the joint probability can be factorized as
{
    \small
    \begin{equation*}
        \begin{aligned}
              &\; p(\bar{\stateSeq}, \bar{\inputSeq}, \bar{\modeSeq}, z_{\tau} \!\midsc\! \theta) \\
            = &\; p(x_{\tau+1}, \bar{u}_{\tau:\tau+1}, \xi_{\tau:\tau+1}, z_{\tau}\!) \\
            &\; \;\;\; p(\bar{\stateSeq}_{\tau{+}2: T{-}1}, \bar{\inputSeq}_{\tau{+}2: T{-}1}, \bar{\modeSeq}_{\tau{+}2: T{-}1}\!\mid\! x_{\tau{+}1}, \bar{u}_{\tau:\tau{+}1}, \xi_{\tau:\tau{+}1}, z_{\tau} \!\midsc\! \theta) \\
            = &\; p(x_{\tau+1}, \bar{u}_{\tau:\tau+1}, \xi_{\tau:\tau+1}, z_{\tau}\!)\\
            &\; \; \textstyle\prod\limits_{t = \tau+1}^{T-2} \!p(x_{t+1}, \bar{u}_{t+1}, \xi_{t+1} \!\mid\! \bar{\stateSeq}_{\tau+1: t}, \bar{\inputSeq}_{\tau: t}, \bar{\modeSeq}_{\tau: t}, z_{\tau} \!\midsc\! \theta).
        \end{aligned}
    \end{equation*}    
}%
The conditional distribution in the second equation can be factorized as 
{
    \small
    \begin{equation*}
        \begin{aligned}
            &\; p(x_{t+1}, \bar{u}_{t+1}, \xi_{t+1} \mid \bar{\stateSeq}_{\tau+1: t}, \bar{\inputSeq}_{\tau: t}, \bar{\modeSeq}_{\tau: t}, z_{\tau} \midsc \theta)\\
            = &\; p(\bar{u}_{t+1}, \xi_{t+1} \mid \bar{\stateSeq}_{\tau+1: t+1}, \bar{\inputSeq}_{\tau: t}, \bar{\modeSeq}_{\tau: t}, z_{\tau} \midsc \theta)\,
                  p(x_{t+1} \mid x_t, \bar{u}_t) \\
            = &\; p(\bar{u}_{t+1} \mid z_{t+1}, \xi_{t+1} \midsc \dBeta)\,
                  p(\xi_{t+1} \mid z_{t+1}, \xi_{t}, \midsc \dTheta)\,
                  p(x_{t+1} \mid x_t, \bar{u}_t),
        \end{aligned}
    \end{equation*}
}%
where the first equation follows the system dynamics \eqref{eq: sys_dyn} with $p(x_{t+1} \mid x_t, \bar{u}_t)$ independent of parameter $\theta$,
the second equation follows the definition of the policy model \eqref{eq: policy_estimation} and the definition of $z_t$ in \eqref{eq: input_state_history}.
Particularly,
$p(\xi_{t+1} \mid z_{t+1}, \xi_{t}, \midsc \dTheta)$ follows the switching mechanism \eqref{eq: switching_mechanism},
and $p(\bar{u}_{t+1} \mid z_{t+1}, \xi_{t+1} \midsc \dBeta)$ is specified by the subsystem dynamics \eqref{eq: subsystem_dyn}.
For this last term, the conditional distribution is Gaussian $\gauss(K_{\xi_t} z_t + b_{\xi_t}, \Sigma_{\xi_t})$.
To apply \EMpp, 
we parameterize this Gaussian distribution by defining $\dBeta \dfn \{\beta_1, \dots, \beta_i\}$ with $\beta_i = (C_i, \Lambda_i)$ where $C_i = \Sigma_i^{-1}\bmat{K_i & b_i}$ and $\Lambda_i = \Sigma^{-1}_i$.
With this parameterization, 
\EMpp applies the MM principle by iteratively constructing and minimizing a surrogate function of the regularized negative log-likelihood \eqref{eq: mle_problem}.
Following \cite[Lemma 3.1, Proposition 3.2, and Proposition 5.1]{wang2024em++},
the surrogate function is defined as 
\begin{subequations}\label{eq: surrogate-explicit}
    \begin{align}
        \Qi[1][k]{\dTheta}
        \dfn \!            &\!
        \tlsum_{t=\tau}^{T-2}\!\tlsum_{i,j\in \Xi}\!\frac{1}{T {-} \tau {-} 1}\!q_{t}^{i,j}\!(\theta^k)\!
        \Big(
            \!\lse\big(\Theta_{i}^\top \!z_t\big) {-} \Theta_{i, j}^\top z_t\!
        \Big), \label{eq: switch-surrogate-explicit}\\
        \Qi[2][k]{\dBeta}
        \dfn  
        & \tilde{c}_{ \theta^k} {+}
        \tlsum_{t=\tau}^{T-2}\!\tlsum_{i \in \Xi}\!\frac{1}{2(T {-} \tau {-} 1)}\! q_{t+1}^i\!(\theta^k)\!%
             \norm{\Lambda_i u_{t+1} {-} C_i \bsmat{z_t \\ 1}}^2_{\Lambda_i^{-1}} 
        \nonumber \\
        & \phantom{\tilde{c}_{ \theta^k}} {-}%
        \tlsum_{t=\tau}^{T-2}\!\tlsum_{i \in \Xi}\frac{1}{2(T {-} \tau {-} 1)} q_{t+1}^i( \theta^k)
            \ln\det(\Lambda_i)  
        \label{eq: subsys-surrogate-explicit}                  
    \end{align}
\end{subequations}
where 
$q_{t}^{i,j}( \theta^k) = p(\xi_{t} = i, \xi_{t+1} = j \mid \bar{\stateSeq}, \bar{\inputSeq}, z_\tau \midsc \theta^k)$, 
$q_{t+1}^i( \theta^k) = p(\xi_{t+1} = i \mid \bar{\stateSeq}, \bar{\inputSeq}, z_\tau \midsc \theta^k )$ are posterior distribution derived from the majorization step using \cite[Proposition 5.2]{wang2024em++},
and $\tilde{c}_{\theta^k}$ is a constant.
By choosing a convex regularization term $\reg(\theta) = \reg_1(\dTheta) + \reg_2(\dBeta)$, where
\begin{subequations}\label{eq: reg_exp}
\begin{align}
    \reg_1(\dTheta) = &\; \tlsum_{i = 1}^d \frac{\gamma_1}{2} \norm{\Theta_i}^2_{\mathrm{F}}, \\
    \reg_2(\dBeta) = &\; \frac{1}{2}\tlsum_{i = 1}^d \gamma_2 \left[
            \tr(\Lambda_i) {-} \ln\det(\Lambda_i) 
        \right] {+} \gamma_3 \norm{C_i}^2_{\Lambda_i^{-1}}
\end{align}
\end{subequations}
and $\gamma_1, \gamma_2, \gamma_3 > 0$,
the surrogate problem 
\begin{equation}\label{eq: surrogate_problem} 
    \minimize_{\theta} \Qi[1][k]{\dTheta} + \Qi[2][k]{\dBeta} + \reg(\theta)
\end{equation}
is a convex optimization problem.

%% file: contents/extensions/stable_system_identification.tex
The primary goal of identifying a parameterized model \eqref{eq: policy_estimation} is for multi-step ahead prediction.
Due to its recurrent structure on the control input $u_t$, 
as shown in \eqref{eq: policy_estimation} and \eqref{eq: input_state_history},
a stable system \eqref{eq: policy_estimation} is desired to prevent unbounded growth in the multi-step ahead prediction.

We assume that $z_t$ is linear w.r.t. $u_t$.
For notational simplicity, we set $t_u = 1$ and $z_t = \bmat{u_{t-1}& x_t}$.
Under this parameterization, 
the model \eqref{eq: policy_estimation} is in the form:
\begin{subequations}\label{eq: policy_esti_stable}
    \begin{align}
        \xi_t \sim &\; \sigma_{\xi_t}(\Theta_{\xi_{t-1}}^\top z_t)\\
        u_t =&\; A_{\xi_t} u_{t-1} + B_{\xi_t} x_t + \tilde{w}_t,
    \end{align}
\end{subequations}
where $K_{\xi_t} = \bmat{A_{\xi_t} & B_{\xi_t}}$ and $\tilde{w}_t \sim \gauss(b_{\xi_t}, \Sigma_{\xi_t})$.
Initialized with $u_0$ and $x_1$,
the $N$-step prediction of the control input is computed via the forward recursion:
\begin{equation*}
    \begin{aligned}
        \hat{u}_1 = &\; A_{\xi_1} u_{0} + B_{\xi_1} x_1 + \tilde{w}_0, \\
        \hat{u}_2 = &\; A_{\xi_2} \hat{u}_{1} + B_{\xi_2} \hat{x}_2 + \tilde{w}_2, \\
        \vdots &\; \\
        \hat{u}_N = &\; A_{\xi_N} \hat{u}_{N-1} + B_{\xi_N} \hat{x}_N + \tilde{w}_N,
    \end{aligned}
\end{equation*}
where $\hat{x}_{t}$ is the state prediction by employing the dynamic model $\hat{x}_{t+1} = f(\hat{x}_t, \hat{u}_t)$ for $t = 1, \dots, N-1$.
To prevent unbounded growth of control inputs in the multi-step ahead prediction,
it is necessary to identify a system \eqref{eq: policy_esti_stable} that exhibits stability w.r.t. $u_{t-1}$.
The stability condition is enforcing a Lyapunov inequality.
Namely, there exists a matrix $P\in \mathbb{S}_{++}^{n_u}$, such that 
for all $i = 1, \dots, d$,
\begin{equation}\label{eq: lyapunov_cons}
    A_i^\top P A_i - P \prec 0.
\end{equation}

Recall from \cref{sec: policy_learning},
the variables in the surrogate problem \eqref{eq: surrogate_problem} are obtained by performing a change of variables.
To incorporate the constraint \eqref{eq: lyapunov_cons} into the problem \eqref{eq: mle_problem} and the surrogate problem \eqref{eq: surrogate_problem},
it is necessary to perform the same change of variables for the Lyapunov inequality \eqref{eq: lyapunov_cons}.
The following lemma provides a sufficient condition on the new variables to satisfy inequality \eqref{eq: lyapunov_cons}.

\begin{lemma}\label{lem: lmi}
    Consider system \eqref{eq: policy_esti_stable}.
    Let $C_i = \Sigma_i^{-1} \bmat{K_i & b_i} = \Sigma_i^{-1} \bmat{A_i & B_i & b_i}$, 
    $\Lambda_i = \Sigma_i^{-1}$.
    Define a linear selection operator $S(C_i) = \Sigma_i^{-1}A_i$ that selects the first $n_u$ columns of matrix $C_i$.
    The inequality 
    \begin{equation}\label{eq: lmi_stable}
        \bmat{
            P & S(C_i)^\top \\
            S(C_i) & 2\Lambda_i - P
        }
        \succ 0, \quad \forall i = 1, \dots, d,
    \end{equation}
    implies \eqref{eq: lyapunov_cons}.
\end{lemma}
\begin{proof}
    The first term in \eqref{eq: lyapunov_cons} can be expressed as 
    \[
        A_i^\top P A_i 
        = A_i^\top\big(\Sigma_i^{-1}\Sigma_i\big) P \big(\Sigma_i \Sigma_i^{-1}\big)A_i 
        = S(C_i)^\top \Sigma_i P \Sigma_i S(C_i).
    \]
    Hence, the inequality \eqref{eq: lyapunov_cons} is equivalent to
    \begin{equation}\label{eq: schur_compliment_based_inequality}
        S(C_i)^\top \Lambda^{-1}_i P \Lambda^{-1}_i S(C_i) - P \prec 0.
    \end{equation}
    Applying \cite[Sec. 2.4.3.6]{caverly2019lmi},
    \eqref{eq: lmi_stable} implies \eqref{eq: schur_compliment_based_inequality},
    thus implies \eqref{eq: lyapunov_cons}.
\end{proof}

Having established the stability conditions through the LMI formulation \eqref{eq: lmi_stable},
we can now incorporate these constraints into the minimization step of \EMpp.
Specifically, while the switching surrogate function \eqref{eq: switch-surrogate-explicit} remains unchanged,
the subsystem surrogate problem \eqref{eq: subsys-surrogate-explicit} is modified to include the Lyapunov stability constraints \eqref{eq: lmi_stable}.
Thus, to summarize, in the minimization step, 
we solve problem
\begin{equation}
    \begin{aligned}
        \minimize_{\theta, P}  \Qi[1][k]{\dTheta} \;+&\; \Qi[2][k]{\dBeta} + \reg(\theta) \\
        \stt. \quad\quad\quad\quad 
        P \succ &\; 0, \\
         \bmat{
            P & S(C_i)^\top \\
            S(C_i) & 2\Lambda_i - P
        }
        \succ &\; 0, \quad \forall i = 1, \dots, d
    \end{aligned}
\end{equation}
where $\Qi[1][k]{}, \Qi[2][k]{}, \reg_2$ are defined in \eqref{eq: switch-surrogate-explicit}, \eqref{eq: subsys-surrogate-explicit}, and \eqref{eq: reg_exp}, respectively.

%% file: contents/autonomous_driving_application/introduction.tex
In this section,
we apply \cref{alg: two_stage} to two autonomous driving datasets. 
The data was collected from human drivers on a Stewart-platform driving simulator in previous work \cite{acerbo2024drivingvisiondifferentiableoptimal}. The ego vehicle, controlled by the human driver, was modelled via Simcenter Amesim as a 15 degree-of-freedom high-fidelity model, comprising models of the chassis, tires, suspensions, powertrain, aerodynamics and steering torque feedback. 
To incorporate the proposed model into the MPC framework,
a dynamical bicycle model is employed.
We first present the details about the control input estimation by solving \eqref{eq: nonlinear_ls},
followed by numerical experiment results on the two datasets.

%% file: contents/autonomous_driving_application/control_input_estimation.tex
We consider a dynamic bicycle model \cite{allamaa2022real, kong2015kinematic}:
\begin{equation}\label{eq: dynamic_bicycle_model}
    \dot{x} = f(x,u) = 
    \begin{bmatrix}
    v_x \cos \psi - v_y \sin \psi \\
    v_y \cos \psi + v_x \sin \psi \\
    \omega \\
    a + v_y\omega - \frac{1}{M} F_{y,f} \sin \delta \\
    -v_x\omega + \frac{1}{M} (F_{y,f} \cos \delta + F_{y,r}) \\
    \frac{1}{I_Z} (l_f F_{y,f} \cos \delta - l_r F_{y,r})
    \end{bmatrix},
\end{equation}
where the system state 
$x = \bmat{p_X & p_Y & \psi & v_x & v_y & \omega}^\top$ 
denotes the vehicle's $X, Y$ coordinates, the yaw angle in the world frame,
the longitudinal velocity and the lateral velocity in the vehicle body frame, and the yaw rate.
The control input is given by $u = \bmat{a & \delta}$,
where $a$ is the longitudinal acceleration and $\delta$ is the steering angle.
In particular, we consider that $a = T_w / (M R_w)$ where $T_w$ is the longitudinal drive torque from the ground, 
$M$ denotes the vehicle mass,
$R_w$ is the rolling radius of the drive wheel.
In addition, $I_z$ is the inertial about the $Z$ axis, 
$l_f$, $l_r$ are the distances from the center of gravity to the front and rare axle, respectively.
We employ a linear tire model to approximate the lateral forces, utilizing cornering stiffness values $C_f$ and $C_r$ for the front and rear tires:
\begin{equation*}
    \begin{aligned}
        F_{y, i} = &\; 2 C_{i} \alpha_i, \quad i \in \{f, r\}, \\
        \alpha_f = &\; \delta - \arctan(\frac{\omega l_f + v_y}{v_x}),\\
        \alpha_r = &\; \arctan(\frac{\omega l_r - v_y}{v_x}).
    \end{aligned}
\end{equation*}

To estimate the steering angle, 
we employ the dynamics of $\omega$ and solve the equation 
\[
    \dot{\omega} = \frac{1}{I_Z} (l_f F_{y, f} \cos \delta - l_r F_{y,r}),
\]
where we use Euler forward method to estimate the derivative $\dot{\omega}$.
Then, employing the dynamics of $v_x$ defined in \eqref{eq: dynamic_bicycle_model},
we obtain the estimation of acceleration by solving the equation 
\[
    \dot{v}_x = a + v_y\omega - \frac{1}{M} F_{y,f}(\bar{\delta}) \sin \bar{\delta}.
\]

%% file: contents/autonomous_driving_application/numerical_experiments/introduction_and_metric.tex
As the primary goal of learning the policy is for motion prediction,
we validate the performance of the estimated control policy through state trajectory prediction accuracy.
To evaluate the effectiveness of the stability constraint,
we employ the \emph{recursive one-step-ahead prediction}:
the control input $\hat{u}_t$ is estimated using \eqref{eq: policy_estimation} with the previously predicted control inputs and the ground truth state history:
\[z_t = (\hat{u}_{t-1}\dots, \hat{u}_{t-t_u}, x_t^{\mathrm{true}}, \dots, x_{t-t_x}^{\mathrm{true}}),\]
followed by one-step ahead prediction $\hat{x}_{t+1} = f(x_t^{\mathrm{true}}, \hat{u}_t)$.
To integrate the model into the MPC framework and for comparison against other methods,
we employ the \emph{joint input-state prediction}:
the control input $\hat{u}_t$ is estimated using \eqref{eq: policy_estimation} with previously predicted inputs and states:
\[z_t = ( \hat{u}_{t-1}\dots, \hat{u}_{t-t_u}, \hat{x}_t, \dots, \hat{x}_{t-t_x}),\]
followed by the one-step ahead prediction $\hat{x}_{t+1} = f(\hat{x}_t, \hat{u}_t)$.

We compare our method against several approaches.
As baseline comparison,
we consider two simplified variants of the switching mechanism:
\begin{inlinelist}
    \item Mode-dependent switching \eqref{eq: only_mode} (\texttt{only-mode}), and 
    \item state-dpendent switching \eqref{eq: only_state} (\texttt{only-state}).
\end{inlinelist}
The behavioral cloning from observation (BCO(0)) framework \cite{torabi2018behavioral} serves as another comparative baseline, 
along with a constant control input (CC) over the prediction horizon---a conventional approach in short-term MPC prediction applications.
Furthermore, 
since the proposed method explores the hierarchical structure with a controller and a system, 
we evaluate it against an end-to-end approach, 
where the states are directly predicted by the previous state history via an LSTM network \cite{hochreiter1997long}.
The performance of each method is evaluated
by the mean absolute error (MAE):
\[ 
    \mathrm{MAE}_i = \tfrac{1}{N} \tlsum_{t = 0}^{N-1}\abs{\hat{x}_{i, t+1} - x^{\mathrm{true}}_{i, t+1}}, \quad\forall i = 1, \dots, n_x,
\]
where $N$ represents the trajectory length.

%% file: contents/autonomous_driving_application/numerical_experiments/lane_keeping.tex
\subsubsection{Lane-keeping scenario}\label{sec: lane_keeping}
We firstly evaluate the proposed method on a lane-keeping scenario. 
The human driver is driving on a $\SI{4.5}{\meter}$ wide track with 8 clothoidal curves, and controls throttle, brake and steering wheel angle.
In this scenario, five trajectories are collected with a sampling rate of $\SI{1}{\kilo\hertz}$.
Of these trajectories,
three trajectories are used for training, one trajectory for validation and one for testing.
We down-sample all trajectories and the estimated inputs $\{\bar{u}_t\}_{t=0}^{T-1}$ presented in \cref{sec: input_estimation_dynamic_bicycle} with $\SI{20}{\hertz}$,
resulting in 2200 data points for each trajectory.

To construct the input-state history \eqref{eq: input_state_history},
we choose $z_t = (\bar{u}_{t-1}, x_{t-1}, x_t, m_{t-1}, m_t)$,
where the map information
$m_t = (
    \omega^{\mathrm{c}}_{0, t}, \abs{\omega^{\mathrm{c}}_{0, t}}, \dots, \omega^{\mathrm{c}}_{3, t}, \abs{\omega^{\mathrm{c}}_{3, t}}
).$
Here, $\omega^{\mathrm{c}}_{i, t}$ represents the centerline yaw rate at time $t$, measured at a point $i$ meters ahead of the current vehicle position, where $i = 0, 1, 2, 3$.
In addition,
we transform the vehicle state into the Frenet coordinates \cite{qian2016hierarchical}.
Frenet coordinates describe the vehicle pose with respect to the centerline of the road.
As illustrated in \cref{fig: frenet}, the Frenet coordinate system consists of the arc length $\sigma$, which represents the travel distance along the road;
the lateral offset $d$,
and the heading error $\varphi$ between the vehicle yaw angle and the heading of the road.
We use $x_t = \bmat{v_{x, t}, v_{y,t}, \psi_t, \frac{\sigma_t}{\sigma_{\max}}, d_t, \varphi_t}$ to construct $z_t$,
where $\sigma_{\max}$ is the maximum value of the arc length $\sigma$.

\begin{figure}[!htpb]
    \centering
    \includegraphics[width=0.3\textwidth]{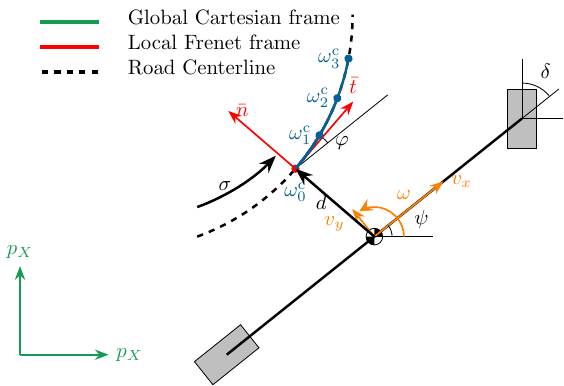}
    \caption{Illustration of Frenet coordinates}\label{fig: frenet}
\end{figure}

We set the regularization parameters $\gamma_1 = \gamma_2 = \gamma_3 = \SI{5e-6}{}$ in \eqref{eq: reg_exp}
and choose the number of modes equal to $3$.
Since the problem \eqref{eq: mle_problem} is nonconvex, 
we solve the problem 10 times with random initializations.
The parameter that performs best on the validation set is selected for the final model.
As the proposed method is a stochastic policy that outputs a distribution,
we sample 100 trajectories from the identified distribution.
To reduce the impact of outliers,
we compute the $1\%$ trimmed mean of these sampled trajectories for the final evaluation metric. 

To evaluate the performance of the imposed stability constraint,
we employ two initialization schemes:
\begin{inlinelist}
    \item We initialize the model \eqref{eq: policy_estimation} following \cite[Proposition 5.2]{wang2024em++}, as suggested in \cite{wang2024em++}; and
    \item we initialize the model \eqref{eq: policy_estimation} with a uniform distribution for the discrete variable $\xi_0$, and with the initial control input $\hat{u}_{0} = \bmat{-1.5, -0.04}$.
\end{inlinelist}
The recursive one-step ahead prediction result is illustrated in \cref{fig: LK_one_step}.
\begin{figure}[!htpb]
    \centering
    \includegraphics[width=0.4\textwidth]{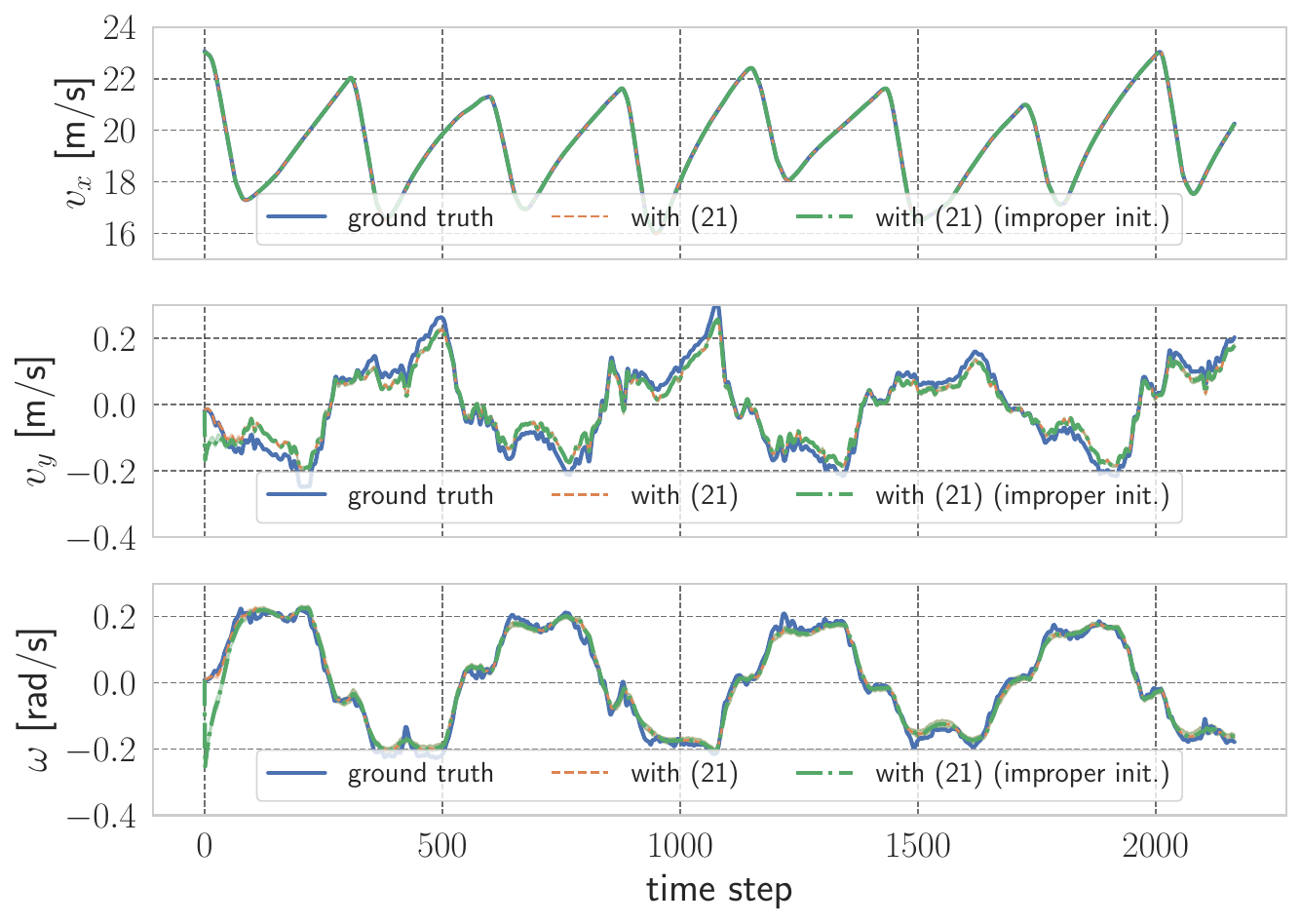}
    \caption{Recursive one-step-ahead prediction for lane-keeping scenario. 
        Shaded regions represent quantiles between $0.25$ and $0.75$. 
    \texttt{Improper init} denotes the case where $p(\xi_0)$ follows a uniform distribution and $\hat{u}_{0} = \bmat{-1.5, -0.04}$.}
    \label{fig: LK_one_step}
\end{figure}

As shown in \cref{fig: LK_one_step},
the properly initialized policy consistently achieves high prediction accuracy.
In contrast,
the improperly initialized policy exhibits large prediction errors in both the lateral velocity $v_y$ and yaw rate $\omega$ initially.
These errors diminish over time.
The prediction error for longitudinal velocity remains consistently low throughout the simulation, 
which can be attributed to the low weights assigned to previous control inputs in the model. 
The convergent behavior indicates that the method effectively mitigates error accumulation, 
validating the effectiveness of the imposed stability constraint.
Note that while training without the stability constraint \eqref{eq: lmi_stable} occasionally yields stable behavior for this particular dataset due to its inherent dynamics,
this stability is not guaranteed.
This will be demonstrated in the subsequent double-lane-change scenario, 
where the absence of the stability constraint consistently leads to unstable predictions.

Next, we evaluate the performance of joint input-state prediction.
For the BCO implementation,
we employ the method presented \cref{sec: input_estimation_dynamic_bicycle} to estimate the control input,
then utilizing a multilayer perceptron (MLP) with 2 hidden layers of 128 neurons for policy learning. 
The network uses the same $\{z_t\}_{t=0}^{T-1}$ with $\texttt{silu}$ activation function.
For the end-to-end approach, we use one LSTM cell with 64 hidden states connected to a linear layer.
The LSTM processes $(x_{t-1}, x_t, m_{t-1}, m_{t-1})$ as input to predict the longitudinal velocity $v_{x, t+1}$,
lateral velocity $v_{y, t+1}$, and the yaw rate $\omega_{t+1}$.
The global coordinates $p_X, p_Y$, and the yaw angle $\psi$ are then integrated using the predicted first order information via Euler-forward method.
The testing trajectory is divided into small segments, each consisting of 100 time steps.
Since both the proposed method and LSTM incorporate latent variables,
we allocate the first 30 time steps for method initialization, 
followed by 60 steps (corresponds to $\SI{3}{\second}$) for prediction.
The mean and standard deviation over all segments are summarized in \cref{tab: MAE_lane_keeping}.
\begin{table*}[!htbp]
    \centering
    \caption{MAE of joint prediction for lane-keeping scenario ($\downarrow$). The best result is marked in bold}
    \label{tab: MAE_lane_keeping}
    \scalebox{0.9}{
    \begin{tabular}{l llllll}
        \toprule \\[-2pt]
                                                &    $p_X [m]$           & $p_Y [m]$          &  $\psi [\mathrm{rad}]$ & $v_x [m / s]$      & $v_y [m/s]$         & $\omega [\mathrm{rad} / s]$ \\
        \midrule                                                              \\
        \EMpp \cite{wang2024em++}               &    $ \boldsymbol{0.273} (\pm 0.228)$  & $ \boldsymbol{0.307} (\pm 0.364)$ & $ 0.010 (\pm 0.005)$       & $ \boldsymbol{0.484} (\pm 0.494)$   & $ 0.048 (\pm 0.017)$   & $ \boldsymbol{0.014} (\pm 0.007)$         \\
        \EMpp (\texttt{only-state})             &    $ 0.286 (\pm 0.355)$  & $ 0.309 (\pm 0.379)$ & $ \boldsymbol{0.009} (\pm 0.005)$       & $ 0.523 (\pm 0.587)$   & $ 0.049 (\pm 0.017)$   & $ \boldsymbol{0.014} (\pm 0.009)$         \\
        \EMpp (\texttt{only-mode})              &    $ 0.413 (\pm 0.269)$  & $ 0.433 (\pm 0.278)$ & $ \boldsymbol{0.009} (\pm 0.006)$       & $ 0.749 (\pm 0.393)$   & $ \boldsymbol{0.040} (\pm 0.016)$   & $ 0.015 (\pm 0.006)$         \\
            
        BCO(0) \cite{torabi2018behavioral}      &    $0.359 (\pm 0.289)$  & $0.385 (\pm 0.337)$ & $0.014 (\pm 0.008)$       & $0.562 (\pm 0.482)$   & $0.050 (\pm 0.020)$   & $0.018( \pm 0.018)$          \\
        CC                                      &    $0.577 (\pm 0.502)$  & $0.615 (\pm 0.665)$ & $0.046 (\pm 0.041)$       & $0.653 (\pm0.780)$   & $0.069 (\pm 0.026)$   & $0.048( \pm 0.040)$          \\
        \midrule\\[-2pt]
        End-to-end (LSTM \cite{hochreiter1997long})                       &    $0.806 (\pm 0.800)$  & $0.675 (\pm 0.518)$ & $0.013 (\pm 0.008)$       & $0.980 (\pm 0.676)$   & $0.043 (\pm 0.021)$   & $0.019( \pm 0.009)$          \\                
        \bottomrule
    \end{tabular}
    }
\end{table*}

From \cref{tab: MAE_lane_keeping}, we observe that the model \eqref{eq: policy_estimation}, trained by \EMpp, outperforms the baselines.
This enhanced performance demonstrates the model's capacity to effectively capture the nonlinear human behavior.
While the simplified switching mechanisms showed advantages in specific state predictions,
overall, the comprehensive switching mechanism \eqref{eq: switching_mechanism} achieves higher prediction accuracy.
Although the LSTM demonstrates comparable accuracy in predicting the lateral velocity $v_y$ and yaw rate $\omega$,
the longitudinal velocity prediction error is significantly higher,
resulting in large deviation in the global coordinates.
These results validate the superior performance of the hierarchical approach.

%% file: contents/autonomous_driving_application/numerical_experiments/double_lane_change.tex
\subsubsection{Double-lane-change scenario}
We secondly evaluate the method on a double-lane-change scenario.
The human driver is driving behind a vehicle on a straight road with a randomly initialized distance.
The front vehicle is driving along the lane center.
At the beginning of the simulation, 
the front vehicle starts decelerating at a randomly selected constant rate of either $a = \SI{-1}{\meter / \second^2}$ or $a = \SI{-0.5}{\meter / \second^2}$.
This deceleration continues until the front vehicle comes to a complete stop.
In this scenario, $21$ trajectories are collected with a sampling rate of $\SI{1}{\kilo\hertz}$. 
Of these trajectories, 14 trajectories are used for training, 3 trajectories are used for validation and 4 trajectories are for testing.
We down-sample the trajectories and the estimated inputs $\{\bar{u}_t\}_{t=0}^{T-1}$ presented in \cref{sec: input_estimation_dynamic_bicycle} with $\SI{20}{\hertz}$, 
resulting in 400 data points for each trajectory.

To construct the input-state history \eqref{eq: input_state_history},
we choose $z_t = (\bar{u}_{t-1}, x_{t-1}, x_t, m_{t-1}, m_{t}, x^{\mathrm{rel}}_{t-1}, x^{\mathrm{rel}}_t)$,
where we use $x_{t} = \bmat{v_{x,t}, v_{y,t}, \psi_t}$ for the ego vehicle.
The map information $m_t = \bmat{d_t^{1}, d_t^{2}, w_{\mathrm{lane}}}$ is a stack of distance to the lane center $d_t^1, d_t^2$, 
and the lane width $w_{\mathrm{lane}} = \SI{4.5}{\meter}$.
The relative states $x_t^{\mathrm{rel}} = \bmat{p_{X, t}^{\mathrm{rel}}, p_{Y, t}^{\mathrm{rel}}, \psi_t^{\mathrm{rel}}, v_{x, t}^{\mathrm{rel}}, v_{y, t}^{\mathrm{rel}}}$ 
denotes the relative distance, yaw angle, the longitudinal and lateral velocity at the front vehicle's body frame.

We set the regularization parameters $\gamma_1 = \gamma_2 = \gamma_3 = \SI{5e-6}{}$ in \eqref{eq: reg_exp}
and choose the number of modes equal to $4$.
Since the problem \eqref{eq: mle_problem} is nonconvex, 
we solve the problem 10 times with random initializations.
The parameter that performs best on the validation set is selected for the final model.
We again use the $1\%$ trimmed mean over 100 sampled trajectories for evaluation. 

We evaluate the performance of the imposed stability constraint using the previously described initialization schemes:
\begin{inlinelist}
    \item The model \eqref{eq: policy_estimation} is initialized following \cite[Proposition 5.2]{wang2024em++}; and
    \item the model \eqref{eq: policy_estimation} is initialized with a uniform distribution for the discrete variable $\xi_0$ and $\hat{u}_{0} = \bmat{-0.05, -0.01}$.
\end{inlinelist}
For comparative analysis, we also evaluate the policy trained without the stability constraint \eqref{eq: lmi_stable}.
The recursive one-step ahead predictions are shown in \cref{fig: DLC_one_step}.
\begin{figure}[!htpb]
    \centering
    \includegraphics[width=0.4\textwidth]{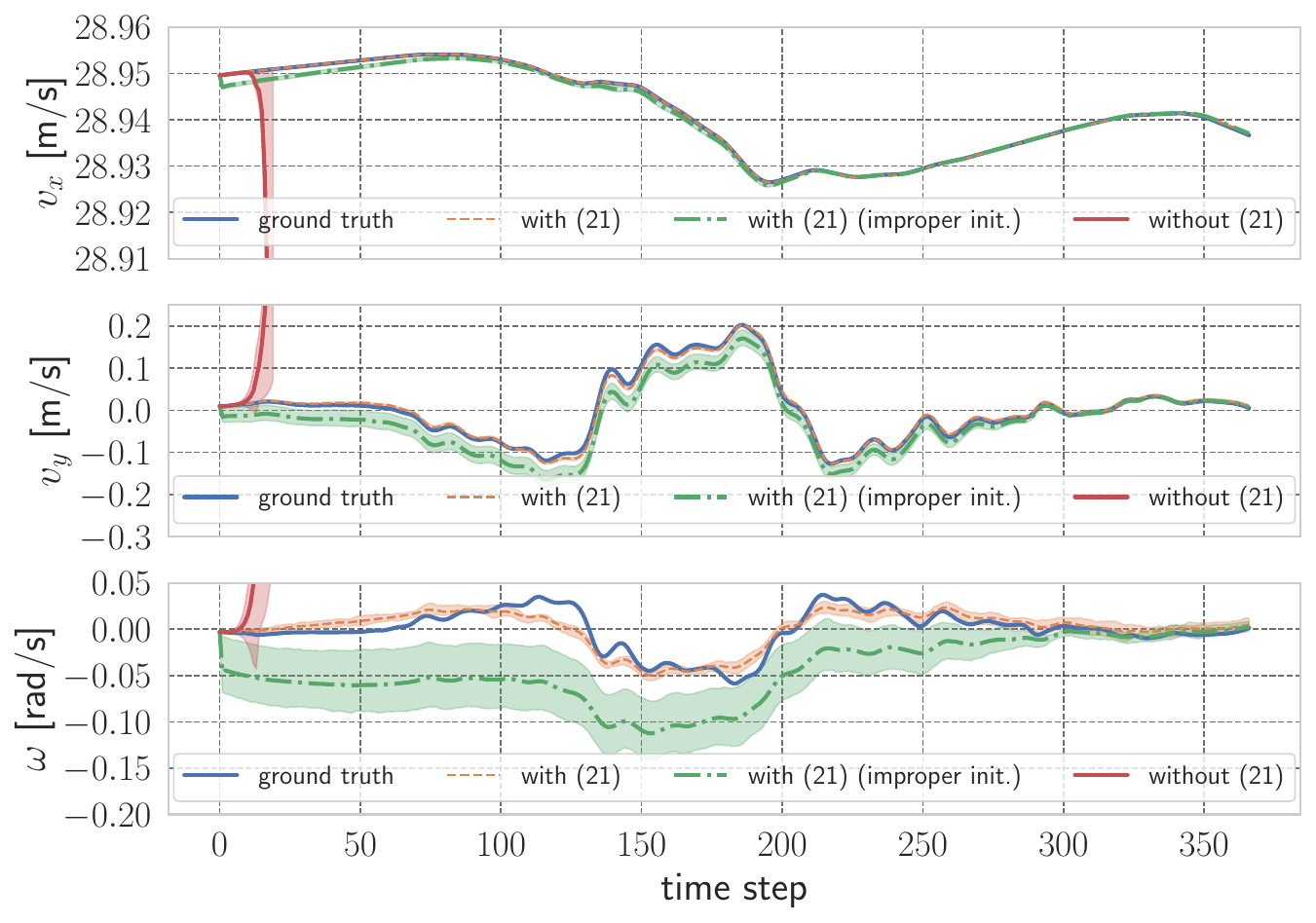}
    \caption{Recursive one-step-ahead prediction for double-lane-change scenario. 
    Shaded regions represent quantiles between $0.25$ and $0.75$.
    \texttt{Improper init} denotes that $p(\xi_0)$ follows a uniform distribution and $\hat{u}_{0} = \bmat{-0.05, -0.01}$.
    For visual clarity, only the first 20 steps of the case \texttt{without \eqref{eq: lmi_stable}} are shown due to significant trajectory oscillations.}
    \label{fig: DLC_one_step}
\end{figure}

As illustrated in \cref{fig: DLC_one_step},
the predictions from the policy trained without constraint \eqref{eq: lmi_stable} rapidly diverge,
demonstrating the necessity of incorporating the stability constraint in the training process. 
In contrast,
the properly initialized policy maintains consistently accurate prediction throughout the simulation.
While the improperly initialized policy exhibits large deviations in all three states initially,
these deviations steadily decrease as the simulation progresses.
This convergent behavior further validates the effectiveness of the imposed stability constraint.

For the joint input-state prediction evaluation,
both BCO and LSTM use the same structure presented in \cref{sec: lane_keeping},
with their networks taking $(x_{t-1}, x_t, m_{t-1}, m_{t}, x^{\mathrm{rel}}_{t-1}, x^{\mathrm{rel}}_t)$ as input.
Same as \cref{sec: lane_keeping},
the testing trajectory is divided into 100-step segments, 
using the first 30 steps for initialization and the subsequent 60 steps (corresponds to $\SI{3}{\second}$) for prediction.
The mean and standard deviation over all segments are summarized in \cref{tab: MAE_double_lane_change}.
    \begin{table*}[!htpb]
        \centering
        \caption{MAE of joint prediction for double-lane-change scenario ($\downarrow$). The best result is marked in bold}
        \label{tab: MAE_double_lane_change}
        \scalebox{0.9}{
        \begin{tabular}{l llllll}
            \toprule \\[-2pt]
                                                    &    $p_X [m]$           & $p_Y [m]$          &  $\psi [\mathrm{rad}]$ & $v_x [m / s]$      & $v_y [m/s]$         & $\omega [\mathrm{rad} / s]$ \\
            \midrule                                                              \\
            \EMpp \cite{wang2024em++}               &    $ \boldsymbol{0.004} (\pm 0.006)$  & $ \boldsymbol{0.180} (\pm 0.125)$ & $ \boldsymbol{0.006} (\pm 0.004)$       & $ \boldsymbol{0.001} (\pm 0.001)$   & $ \boldsymbol{0.026} (\pm 0.015)$   & $ \boldsymbol{0.008} (\pm 0.004)$         \\
            \EMpp (\texttt{only-state})             &    $ 0.009 (\pm 0.008)$  & $ 0.279 (\pm 0.119)$ & $ 0.014 (\pm 0.005)$       & $ 0.003 (\pm 0.003)$   & $ 0.048 (\pm 0.027)$   & $ 0.017 (\pm 0.010)$         \\
            \EMpp (\texttt{only-mode})              &    $ 0.009 (\pm 0.007)$  & $ 0.284 (\pm 0.135)$ & $ 0.014 (\pm 0.005)$       & $ 0.003 (\pm 0.003)$   & $ 0.049 (\pm 0.029)$   & $ 0.017 (\pm 0.011)$         \\
            
            BCO(0): \cite{torabi2018behavioral}     &    $ 0.007 (\pm 0.007)$  & $ 0.213 (\pm 0.163)$ & $ 0.007 (\pm 0.005)$       & $ 0.004 (\pm 0.002)$   & $ 0.029 (\pm 0.020)$   & $ \boldsymbol{0.008} (\pm 0.005)$         \\
            CC                                      &    $ 0.017 (\pm 0.026)$  & $ 0.488 (\pm 0.498)$ & $0.026 (\pm 0.020)$       & $0.003 (\pm0.004)$   & $0.058 (\pm 0.045)$   & $0.018 (\pm 0.014)$         \\
            \midrule\\[-2pt]
            End-to-end (LSTM \cite{hochreiter1997long})                       &    $0.065 (\pm 0.034)$  & $0.206 (\pm 0.117)$ & $0.007 (\pm 0.005)$       & $0.045 (\pm 0.025)$   & $0.028 (\pm 0.016)$   & $\boldsymbol{0.008}( \pm 0.005)$          \\                    
            \bottomrule
        \end{tabular}
        }
    \end{table*}

The performance of \eqref{eq: policy_estimation}, trained by \EMpp,
remains consistent with the observations in \cref{sec: lane_keeping}:
the model \eqref{eq: policy_estimation} achieves superior performance across all variables.
Notably, the complete switching mechanism again demonstrated enhanced overall performance compared to its simplified variants,
validating the necessity of incorporating both state and mode dependence in switching mechanism design.
Consistently to the results in \cref{sec: lane_keeping},
LSTM achieves comparable performance in predicting the lateral velocity $v_y$ and $\omega$ while exhibiting large prediction error in longitudinal velocity $v_x$, 
resulting in overall inferior performance.
The results further substantiate the advantages of the hierarchical approach over the end-to-end method.

%% file: contents/conclusion.tex
In this paper, 
we have presented a stochastic switching system for policy learning from state-only trajectories without requiring action demonstrations. 
A sufficient condition ensuring model stability is derived.
This stability condition can be directly incorporated in the training process.
The proposed approach has been validated on two autonomous driving datasets,
where both the stability and prediction accuracy of the model demonstrate effective results.
Future work includes integrating the proposed model into an MPC framework,
and developing a stochastic version of \EMpp to efficiently handle large-scale datasets while maintaining the established theoretical guarantees.